\documentclass[10pt]{article} 
\usepackage[preprint]{tmlr}

\usepackage{amsmath,amsfonts,bm}

\def\eqref#1{equation~\ref{#1}}

\def\1{\bm{1}}

\DeclareMathAlphabet{\mathsfit}{\encodingdefault}{\sfdefault}{m}{sl}
\SetMathAlphabet{\mathsfit}{bold}{\encodingdefault}{\sfdefault}{bx}{n}

\def\sR{{\mathbb{R}}}

\usepackage{hyperref}
\usepackage{url}

\usepackage{latexsym}
\usepackage[
    ruled,          
    vlined,         
    noend,          
    linesnumbered,  
    resetcount,     
]{algorithm2e}
\newcommand{\removelatexerror}{\let\@latex@error\@gobble}

\usepackage{microtype}

\usepackage{inconsolata}
\usepackage{booktabs} 

\usepackage{graphicx}
\usepackage{amsmath}
\usepackage{amssymb, amsthm}
\newtheorem{theorem}{Theorem}
\newtheorem{corollary}{Corollary}[theorem]

\theoremstyle{definition}

\newtheorem{assumption}{Assumption}

\usepackage{subcaption}

\usepackage{cuted}
\usepackage[center]{formal-grammar} 

\title{Hierarchical Latent Structures in Data Generation Process Unify Mechanistic Phenomena across Scale}

\author{
  \name Jonas Rohweder\thanks{These two authors contributed equally to this work.}
  \email jonas.rohweder@stud.tu-darmstadt.de\\
  \addr UKP Lab, TU Darmstadt\\
  \AND
  \name Subhabrata Dutta\footnotemark[1]
  \email subhabrata.dutta@tu-darmstadt.de\\
  \addr UKP Lab, TU Darmstadt\\
  \AND
  \name Iryna Gurevych 
  \email iryna.gurevych@tu-darmstadt.de\\
  \addr UKP Lab, TU Darmstadt
  }

\def\month{MM}  
\def\year{YYYY} 
\def\openreview{\url{https://openreview.net/forum?id=XXXX}} 

\begin{document}

\maketitle

\begin{abstract}
Contemporary studies in mechanistic interpretability have uncovered many puzzling phenomena in the neural information processing of Transformer-based language models, such as induction heads, function vectors, and the Hydra effect. Some of these individual phenomena have been independently tied to different data distributional properties, while some have been loosely associated with model architecture and how Transformers process information. However, a unified understanding of the relationship between data, model architecture, and optimization remains lacking, failing to answer the fundamental question: why do these three phenomena appear universally across different model families and scales, despite their seeming disconnect? In this work, we answer this question by unifying these three phenomena as consequences of hierarchical latent structures in the data generation process, coupled with decorrelated gradients across additive model components and directional concavity in the representation geometry. We validate our theoretical results in a toy model regime and in a large-scale synthetic data regime, comparing them with language models trained on natural language data.
\footnote{We make our code available at: \url{https://github.com/UKPLab/arxiv2026-hierarchical-latent-structures}.}
\end{abstract}

\section{Introduction}

Recent investigations aiming at representative disassembly of a language model, commonly termed under the broad umbrella of mechanistic interpretability~\citep{olah2022mechanistic}, have identified striking phenomena emerging within these models --- inference-time self-repair or the \textit{Hydra effect}~\cite{mcgrath2023hydra}, \textit{induction heads}~\citep{olssonIncontextLearningInduction2022}, \textit{function vectors}~\citep{todd2023function}. Each of these phenomena is universal to autoregressive language models trained on large natural language corpora; yet, we lack any framework to explain why they appear or an explanation for their coincidental emergence. 

We note that not all mechanistic phenomena act on the same scale/scope. For example, induction circuits are pertinent to local predictive structures of the form ${\tt [A][B]}\cdots {\tt [A]}\rightarrow{\tt [B]}$, where the prediction of ${\tt [B]}$ is conditioned on matching prefix {\tt [A]}. Function vectors are generalizations of induction-like mechanisms, in which the model compresses task-specific information from in-context evidence into reusable vector representations. The Hydra effect, on the other hand, is fundamental to how predictive information is processed in the different parts of the model architecture: ablation of a model component at inference time causes another component to compensate via increased predictive influence.

Behind this seeming disconnect, however, lies some intuitive commonality. The linkage between induction heads and function vectors is fairly straightforward. Both these phenomena explain in-context learning. Recent literature~\citep{chanDataDistributionalProperties2022} attribute the emergence of in-context learning to multiple data-distributional properties, such as burstiness, dynamic meaning, Zipfian distribution, etc. It is natural to assume that these properties control induction heads and function vectors as well. \citet{chanDataDistributionalProperties2022} note an interesting observation: while the mentioned data-specific properties elicit in-context learning, it is exclusive to Transformers and not recurrent architectures.

Prior work posit that \textit{copy suppression} causes Hydra effect: if model components in earlier layers predict a certain
token, and this token appears earlier in the context, certain attention heads (termed copy suppressing heads) suppresses that copied information~\citep{DBLP:journals/corr/abs-2310-04625}. When the suppressed head is ablated, suppression is not in place and hence, the predictive performance does not get hurt. This observation unravels a deeper connection between Hydra effect and in-context learning: for the copy suppression to successfully elicit Hydra effect, the model needs to have similar yet redundant evidence in the context, and must be able to infer useful information from such evidences. Such ability, as we discussed, are inherently linked to induction heads and function vectors.

To this end, we seek to unify the explanations of these three unique mechanistic properties under the umbrella of data distribution process, model architecture, and learning dynamics. The common theme of all three properties is the presence of shared semantics across different parts of the input. We model this shared semantics as hierarchically structured latents in the data generation process. We first develop a theoretical framework to attribute the role of hierarchy in the elicitation of induction heads, function vectors, and the Hydra effect (Section~\ref{sec:theory}). We verify the theoretical results on a strictly controlled toy model regime: Transformers trained on hierarchical data generation process exhibit the properties under consideration significantly better than a flat data generation process with same surface-level statistics (Section~\ref{sec:toy-model}). Finally, in Section~\ref{sec:large-models}, we perform large-scale comparison between Transformers trained on synthetic N-gram data (a flat process), synthetic data generated by Probabilistic Context-Sensitive Grammars or PCSGs (a hierarchical process) and the OLMo-1B checkpoints~\citep{DBLP:conf/acl/GroeneveldBWBKT24} (language model trained on natural language). Confirming the results drawn from the theoretical framework and toy model experiments, models trained on PCSG demonstrate the emergence of these three phenomena, similar to OLMo-1B, while the N-gram model fails to do so.

\newlength{\maxheight}


    
    

\section{Related Work}

{\bf Mechanistic phenomena in language models.} Early works in mechanistic interpretability, primarily popularized by \citet{olah2022mechanistic,elhageMathematicalFrameworkTransformer2021, olssonIncontextLearningInduction2022}, focus on reverse engineering the computation performed by the language model as circuits. Induction circuits~\citep{olssonIncontextLearningInduction2022} are the very first family of such mechanisms found in toy as well as large-scale Transformer-based language models: two-layer circuits that perform prefix-matching and copying. Subsequent work has studied how induction circuits form during training in a synthetic data regime~\citep{DBLP:conf/icml/SinghMHCS24}. Interestingly, they point toward a emergent redundancy of induction heads, though falling short in explaining the redundancy. \citet{chen2024unveiling} prove that gradient flow on Markov models with repeated evidence. \citet{olssonIncontextLearningInduction2022}originally linked induction heads as precursor to in-context learning. Toward identifying the mechanism of few-shot in-context learning, \citet{todd2023function} demonstrate the existence of function vectors (FVs): compact, causally effective representations of in-context examples localized in a small set of attention heads. Independent contemporary work by \citet{hendel-etal-2023-context} confirm the same as task vectors. ~~\citet{dong2026understanding} show that task
vectors naturally emerge in linear transformers trained on triplet-formatted prompts
through loss landscape analysis. Across many of these discussed work, along with orthogonal research into reverse engineering circuits responsible for indirect object identification~\cite{wangInterpretabilityWildCircuit2022}, presence of redundant mechanisms has been pointed out. \citet{mcgrath2023hydra} present a comprehensive analysis of redundancy elicited as Hydra effect: ablating an attention layer causes downstream layers to compensate. \citet{DBLP:journals/corr/abs-2310-04625} provide a partial mechanistic account of Hydra effect via copy suppression heads. \citet{Rushing2024selfrepair} attribute Hydra effect to two different mechanisms within the model: Layernorm scaling and anti-erasure MLP neurons.


{\bf Training dynamics} offer insight into emergent behaviors in neural networks, where gradual increases in scale or training progress lead to sudden qualitative changes in capability \citep{anderson1972more,weiEmergentAbilitiesLarge2022a}. Such phase transitions have been hypothesized to arise from the rapid formation of specialized circuits~\citep{michaudQuantizationModelNeural2024}. Transformer architecture has been shown to possess implicit biases of gradient flow that the near-optimal model is equipped with induction heads~\citep{chen2024unveiling}. \citet{zucchet2025how} study how factual recall is acquired over training. Despite isolated understandings of different such phenomena, a unified picture remains missing. Moreover, phenomena such as Hydra effect~\citep{mcgrath2023hydra} remain unexplored from the training dynamics and data distribution lens. 

{\bf Synthetic data for interpretability.}
Prior work has shown that model performance follows robust scaling laws determined not only by model size and compute but also by the structure of the training data~\citep{kaplanScalingLawsNeural2020,hoffmannTrainingComputeOptimalLarge2022}, highlighting a gap between expressibility and learnability that depends on data structure and optimization dynamics~\citep{weissThinkingTransformers2021,borensteinWhatLanguagesAre2025}. To study how data properties give rise to emergent behaviors, researchers have increasingly turned to synthetic datasets, which allow precise control over statistical and structural features while preserving key learning dynamics observed in natural language models~\citep{elhageToyModelsSuperposition2022,jain2023mechanistically}. However, many synthetic setups rely on flat or sequential data and fail to capture the hierarchical and recursive structure central to language. Probabilistic context-free grammars (PCFGs) offer a principled middle ground, exposing a transparent data generation process while retaining core linguistic properties, such as compositionality, ambiguity, and recursion. Previous work has used PCFGs to probe inductive biases~\citep{whiteExaminingInductiveBias2021}, derive exact comparisons between learned and true distributions~\citep{jumeletTransparencySourceEvaluating2023}, and study the acquisition of hierarchical syntax~\citep{allen-zhuInterpretabilityLanguageModels2024}. Most closely, \citet{schulzUnravelingSyntaxHow2025} show that transformers reduce loss across grammatical components in parallel yet struggle to learn deep recursion, suggesting optimization limits rather than architectural constraints. While these studies establish PCFGs as effective tools for analyzing linguistic competence, they largely stop short of mechanistically explaining how models internalize structure. 

\citet{hu-etal-2025-circuits} demonstrate a strong precedence to our case; they coin the term \textit{pre-pretraining} (training on formal language before natural languages) and found that hierarchical formal language data elicits linguistic generalization in natural language models.

While there have been prior effort to explain some of these mechanistic phenomena from, they remain of narrow focus: either they attribute individual phenomenon to different model components (e.g., induction heads, copy suppression heads, anti-erasure neurons, etc.) and fail to answer why they appear in the training process. Prior work that investigate the role of data distribution and training dynamics usually focus on a single phenomenon in a restrictive setting. On the contrary, we present the first theoretical and empirical unification of seemingly disconnected mechanistic phenomena as an immediate result of hierarchy in the data generation process.

\section{Theoretical Results}
\label{sec:theory}

In this section, we show that under realistic assumptions about the data generation process and model architecture, hierarchically persistent features are guaranteed to \textit{enforce} the emergence of induction, function vectors, and the Hydra effect in an autoregressive model trained via gradient descent.

Let $X = \left(x_1x_2\cdots x_n\right)$ be a sequence of tokens generated by a latent hierarchical process of depth $D\geq2$:
\[Z^{\left(0\right)}\rightarrow Z^{\left(1\right)}\rightarrow X\]
where $Z^{\left(0\right)}$ is a global latent variable, $Z^{\left(1\right)}$ is a local latent variable, and the joint distribution factorizes as:
\[p\left(X,Z^{\left(0\right)},Z^{\left(1\right)}\right)=p\left(Z^{\left(0\right)}\right) p\left(Z^{\left(1\right)}|Z^{\left(0\right)}\right)\prod_{t=1}^n p\left(x_t|Z^{\left(1\right)}_t,Z^{\left(0\right)}\right)\]

We make the following assumptions about the data generation process:
\begin{assumption}
    \label{as::regularity}
    \textit{Regularity.} All random variables are defined on standard Borel spaces, ensuring the existence of regular conditional distributions.
\end{assumption}
\begin{assumption}
\label{as::hier-dependence}
    \textit{Hierarchical dependence.} There exist tokens such that
    \[x_t\not\perp Z^{\left(0\right)}|Z^{\left(1\right)},\;\; x_t \not\perp Z^{\left(1\right)}|Z^{\left(0\right)}\]
    i.e., latents from different levels jointly dictate token probability.
\end{assumption}

\begin{assumption}
\label{as::non-degen}
    \textit{Non-degeneracy.} Latent variables influencing prediction are statistically recoverable from the input:
    \[I\left(Z^{\left(0\right)};X_{<t}\right)>0,\;\; I\left(Z^{\left(1\right)};X_{<t}|Z^{\left(0\right)}\right)>0\]
\end{assumption}
\begin{assumption}
\label{as::multi-evidence}
    \textit{Multiple evidence streams.} There exist disjoint index sets $S_1, S_2, \cdots\subset \{1,\cdots,n\}$ such that,
    \[X_{S_i}\rightarrow Z^{(1)}_i \rightarrow Z^{(0)}, \; \text{for } i=1,2\]
    and
    \[Z^{(1)}_1\perp Z^{(1)}_2 | Z^{(0)}\]
    i.e., two disjoint evidence sets shares latent via distinct local latents.
\end{assumption}

We consider an autoregressive model:
\[p(x_t|X_{<t})={\cal M} \left(X_{<t}\right) \in \sR^{V}\]
where a fixed-dimensional internal representation is mapped to logits via linear decoding $W_U$:
\[{\cal M} \left(X_{<t}\right) = W_Uh^{\left(L\right)}\left(X_{<t}\right)
\]
The final representation $h^{(L)}$ allows additive decomposition into $K$ functional units $f_k$ parametrized by $\theta_k$:
\[h^{(L)}(X) = g\left(\sum_{k=1}^K f_k(X; \theta_k)\right)\]
Each $f_k$ can be loosely described as an information processing path. A smooth, nonlinear function $g$ then aggregates these pathways to retrieve the global latent. In practice, $f_k$ can be either conceptualized via parallel attention heads (where $g$ can be the MLP that aggregates), or via residual connected layers (where $g$ is a subsequent layer that aggregates). We do not commit to either. Instead, for the theoretical argument, we rely on the following assumption:

\paragraph{Bayes optimal predictor} for the defined data generative process can be formulated as:
\begin{equation}
\label{eq:bayes-optimal}
P\left(x_t|X_{<t}\right) = \sum_{z^{(0)}, z^{(1)}_t} 
P\left( x_t| z^{(0)}, z^{(1)}_t\right)
P\left(z^{(1)}_t|z^{(0)}, X_{<t}\right)
P\left(z^{(0)}| X_{<t}\right)
\end{equation}

Define 
\[g\left(z^{(0)}, X_{<t}\right)=\mathbb{E}_{z^{(1)}_t|z^{(0)}, X_{<t}}\left[P\left(x_t|z^{(1)}_t,z^{(0)}\right)\right]\]

Then Eq.~\ref{eq:bayes-optimal} can be rewritten in the following nested form:
\begin{equation}
\label{eq::nested-bayes}
    P\left(x_t|X_{<t}\right) = \mathbb{E}_{z^{(0)}|X_{<t}} \left[g\left(z^{(0)}, X_{<t}\right)\right]
\end{equation}
\begin{theorem}
\label{th::induction}
    Under Assumptions \ref{as::regularity}-\ref{as::multi-evidence}, any finite-capacity model minimizing next-token loss must implement representation that aggregates information about $Z^{(0)}$ from multiple disjoint subsets of the context via intermediate latent structure.
\end{theorem}

\noindent\textit{Proof}: From Eq.~\ref{eq::nested-bayes}, the prediction depends on inferring the global latent $P\left(z^{(0)}|X_{<t}\right)$.

By Assumption~\ref{as::multi-evidence},
\[P\left(z^{(0)}|X_{<t}\right) = \sum_{z^{(1)}_i}P\left(z^{(0)}|z^{(1)}_i\right)P\left(z^{(1)}_i|X_{S_i}\right), \; \text{for }i=1,2\]
and since,
\[Z^{(1)}_1\perp Z^{(1)}_2 | Z^{(0)}\]
The posterior computation requires aggregation over disjoint contexts via different latent pathways.\qed

Theorem~\ref{th::induction} directly connects to the existence of induction heads and computation of function vectors. Induction mechanisms perform the context-matching and retrieval from arbitrary evidences. Function vectors are representation of the posterior estimate $P(Z^{(0)}|X_{<t})$ that are subsequently used to predict $x_t$. 

Next, we proceed to establish the connection between Hydra effect and hierarchical data generation process. 

\begin{theorem}
\label{thm:additive_evidence}







Define the branch evidence function
\[
e_i(z^{(0)},x_i)
= \log
\left(
\sum_{z_i^{(1)}}
p(z_i^{(1)} \mid z^{(0)})
p(x_i \mid z_i^{(1)}, z^{(0)})
\right).
\]
Then the posterior distribution of the global latent satisfies
\[
\log p(z^{(0)} \mid x_{1:m})
=\log p(z^{(0)})
+
\sum_{i=1}^{m}
e_i(z^{(0)},x_i)
-
\log p(x_{1:m}).
\]







Hence each conditionally independent branch contributes an additive
log-evidence term toward inference of the shared latent \(Z^{(0)}\).
\end{theorem}

For the complete proof of Theorem~\ref{thm:additive_evidence}, see Appendix~\ref{app::proof::additive}. We make the following assumption about the gradients associated with these additive branch evidences implemented by the model:

\begin{assumption}[Low-Interference Optimization]
\label{ass:low_interference}
Let
\[
u_i(\theta)
=
\nabla_\theta \hat e_i(z^{(0)},x_i)
\]
denote the gradient associated with the model's estimate
$\hat e_i$ of branch evidence $e_i$.
Assume there exists $\varepsilon > 0$ such that
\[
\left|
\cos(u_i,u_j)
\right|
=
\left|
\frac{u_i^\top u_j}
{\|u_i\|\|u_j\|}
\right|
\le \varepsilon
\]
for all $i\neq j$.

\end{assumption}

Our next theorem posits that such additive, conditionally independent branches result in approximately block-diagonal Fisher matrices.

\begin{theorem}
\label{thm:fisher_decoupling}

Under the assumptions of
Theorem~\ref{thm:additive_evidence}
and Assumption~\ref{ass:low_interference},
let
\[
u
=
\nabla_\theta
\log p_\theta(z^{(0)}|x)
\]
denote the score function of a posterior estimator, which can be written as
\[
u
=
\sum_{i=1}^{m}
u_i\,
\text{where }
u_i
=
\nabla_\theta \hat e_i(z^{(0)},x_i)
\]
corresponds to the learned estimator of the
branch evidence term.
Define the Fisher information matrix
\(
F
=
\mathbb E
\left[
uu^\top
\right]
\), then
\[
F
=
\sum_{i=1}^{m}
F_i
+
R,
\]
where
\[
F_i
=
\mathbb E
\left[
u_i u_i^\top
\right]
\]
and the residual coupling matrix satisfies
\[
\|R\|
\le
\varepsilon
\sum_{i\neq j}
\sqrt{
\mathbb E[\|u_i\|^2]
\,
\mathbb E[\|u_j\|^2]
}
\]
Consequently,
\[
F
=
\sum_i F_i
+
O(\varepsilon),
\]
and therefore approaches block-diagonal form as
$\varepsilon \rightarrow 0$.
\end{theorem}
For the complete proof, see Appendix~\ref{app::subsec::fisher}. An immediate result from Theorem~\ref{thm:fisher_decoupling} is the following corollary:
\begin{corollary}[Emergent Specialization]
\label{cor:specialization}

Under the assumptions of
Theorem~\ref{thm:fisher_decoupling},
the optimization geometry decomposes into
approximately independent Fisher subspaces
associated with the evidence terms
$\{e_i\}_{i=1}^{m}$.
Consequently, gradient updates arising from one
branch induce only $O(\varepsilon)$ interference
on parameters primarily used to estimate another branch.
As $\varepsilon \to 0$,
optimization dynamics become increasingly decoupled,
favoring the emergence of specialized computational
circuits that estimate distinct branch-evidence terms.

\end{corollary}

Now we proceed to show how specialization of implementing $e_i$ via dedicated functional units $f_{k_i}$ elicits Hydra effect. We make the following assumption about the representation geometry:

\begin{assumption}[Directional Concavity of Readout]
\label{as::directional-concavity}
Define the scalar readout for the target token $x_t$ as
\[
\phi(\mu) := \langle g(\mu), w_t \rangle,
\]
where $w_t$ is the row of the unembedding matrix $W_U$ corresponding to $x_t$, and $\mu$ is the final representation.
We assume that $\phi$ is twice differentiable and locally concave along directions induced by redundant latent estimators. In particular, for any representation $\nu$ and any unit contribution $\mu_1$ corresponding to an independent estimator of $Z^{(0)}$,
\[
\mathbb{E}\left[ \mu_1^\top H_\phi(\nu)\, \mu_1 \right] \le 0,
\]
where $H_\phi(\nu)$ is the Hessian of $\phi$ at $\nu$.
\end{assumption}

\begin{theorem}
\label{th::hydra-effect}
Let $f_{k_1}, f_{k_2}$ be functional units that specialize in estimating branched evidences $e_1, e_2$ respectively. Let
\[
\mu_i := f_{k_i}(X), \quad i=1,2,
\]
and define the full representation
\[
\mu := \sum_{k=1}^K f_k(X; \theta_k).
\]
Let $r := \sum_{k \neq k_1, k_2} f_k(X; \theta_k)$ denote the contribution of all other units.

Define the predictive influence for the correct token $x_t$ as
\[
\Delta := \langle g(\mu_1 + \mu_2 + r), w_t \rangle,
\]
and the predictive influence after ablating unit $k_1$ as
\[
\tilde{\Delta}_{k_1} := \langle g(\mu_2 + r), w_t \rangle.
\]

Under Assumptions~\ref{as::multi-evidence}, \ref{ass:low_interference}, and \ref{as::directional-concavity}, it holds that
\[
\mathbb{E}\left[\tilde{\Delta}_{k_1}\right] \ge \mathbb{E}\left[\Delta\right].
\]
\end{theorem}

See Appendix~\ref{app::proof::hydra} for the complete proof.

\section{Toy Model of Hierarchy}
\label{sec:toy-model}

To complement the theoretical results of Theorems~\ref{th::induction}-\ref{th::hydra-effect}, we design a controlled
experimental setup that allows us to test each claim in isolation, free from
confounders present in natural-language corpora.  In this section, we describe the data
generation processes, model, mechanistic metrics, assumption probes, and
significance-testing procedure.

\subsection{Data Generation Processes}

We instantiate two DGPs that share surface statistics but differ in latent
structure, following the theoretical hierarchy of Section~\ref{sec:theory}.

\paragraph{Hierarchical DGP}  Each input sequence of length $n$ is divided into $S$
fixed-length segments of length $L = n / S$.  A global latent variable
$Z^{(0)} \in \{1, \ldots, K_0\}$ is drawn once per sequence from a uniform
prior $\pi_0$.  Each segment $s$ receives an independent local latent
$Z^{(1)}_s \in \{1, \ldots, K_1\}$ drawn from the conditional table
$\pi_1(Z^{(1)} \mid Z^{(0)})$, which is concentrated — each global state
concentrates its mass over a distinct subset of local states, making
$Z^{(1)}$ an informative but noisy signal about $Z^{(0)}$.

Within each segment, the first $L - 2$ positions are \textit{evidence tokens},
drawn from a per-$Z^{(1)}$ emission table $\mathbb{E}_\mathrm{ev}[Z^{(1)}, \cdot]$
that does not depend on $Z^{(0)}$.  Position $L - 1$ is a deterministic
$[\textsc{Query}]$ cue token; position $L$ is the \textit{query token}, drawn from a
joint emission table

$$
\mathbb{E}_\mathrm{q}[Z^{(0)}, Z^{(1)}, x] \;\propto\;
  \alpha_q \cdot E_0[Z^{(0)}, x] + (1 - \alpha_q) \cdot E_1[Z^{(1)}, x],
$$

where $\alpha_q \in (0, 1)$ controls the relative weight of each level.
Query tokens therefore depend \textit{jointly} on both latents, instantiating
Assumption~\ref{as::hier-dependence}: neither $Z^{(0)}$ alone nor $Z^{(1)}$ alone is sufficient for
Bayes-optimal prediction. Optimally predicting query tokens requires
aggregating cross-segment evidence (which reveals $Z^{(0)}$ through the
concentrated transition $\pi_1$) with within-segment evidence (which reveals
$Z^{(1)}$ directly), instantiating disjoint evidence streams of Assumption~\ref{as::multi-evidence}.

\paragraph{Flat DGP}  The flat ablation retains the same segment structure, emission
vocabulary, and surface statistics, but removes the global latent entirely.
Query tokens are drawn from a marginalized emission table
$\mathbb{E}_\mathrm{q,flat}[Z^{(1)}, x] = \sum_{z_0} \pi_0(z_0) \cdot
\mathbb{E}_\mathrm{q}[z_0, Z^{(1)}, x]$, so cross-segment inference is both
unnecessary and impossible.  The model can still exploit within-segment
$Z^{(1)}$ signal, but the Bayes gap $I(Z^{(0)}; x_q \mid Z^{(1)},
X_{< t})$ is positive by construction, guaranteeing a meaningful performance
differential between the two DGPs.

To rule out degenerate configurations, we reject any DGP instance where the
Bayes entropy of the query token decreases by fewer than $0.08$ nats across
context segments (i.e.\ where cross-segment evidence provides negligible
benefit), re-sampling the random seed until this criterion is satisfied.

We train two identical autoregressive Transformers on sequences from
each DGP. Design details are described in Appendix~\ref{app:config}. 

\subsection{Metrics}

We measure three mechanistic phenomena at fixed evaluation intervals during
training.

{\bf Layerwise Hydra effect.} For each upstream layer $\ell'$, we ablate its
output by zeroing the residual contribution and measure the resulting change in
predictive influence at every downstream layer $l > l'$.  The predictive
influence of layer $l$ on sequence $X_{< t}$ is defined as $\Delta^{(l)}
= \langle W_U[x_t],\, h^{(l)}(X_{< t}) \rangle$, where $W_U[x_t]$ is the
row of the unembedding matrix corresponding to the target token.  The Hydra
score for the pair $(l', l)$ is
$$
\mathcal{H}(l', l) \;=\;
  \mathbb{E}\!\left[\Delta^{(l)}_{\text{ablate}(l')} - \Delta^{(l)}\right],
$$
and we report the scalar $\max_{l' < l}\,\mathcal{H}(l', l)$ as the
peak compensation signal.

{\bf Head-wise Hydra effect.} We repeat the same protocol at head granularity:
for each attention head $(i, h)$, we zero its value output before the output
projection and measure the change in downstream predictive influence.  This
yields a three-dimensional tensor $\mathcal{H}[i, h, j] \in \mathbb{R}$ for
all $j > i$, and a summary matrix $\mathcal{H}[i, h] = \max_{j > i}
\mathcal{H}[i, h, j]$ that identifies which heads elicit the strongest
downstream compensation.

{\bf Induction head score.} Following \citet{olssonIncontextLearningInduction2022}, we construct sequences of the form
$[r_0, \ldots, r_{L-1}, A, r_L, \ldots, r_{2L-2}, A]$,
where $A$ is a randomly chosen token and $r_i$ are independent random fillers.
The final token $A$ is the query; the induction target is position $L$
(the token immediately following the first occurrence of $A$).  For head
$(l, h)$, the induction score is the attention weight placed at the induction
target position when processing the final query:
$$
\mathrm{IS}(l, h) \;=\; \mathbb{E}_A\!\left[\alpha^{(l,h)}_{2L-2,\, L}\right],
$$
where the expectation is over stimulus draws.  We report
$\max_{l,h}\, \mathrm{IS}(l, h)$, which is close to $1/T$ for a model with no
induction circuits and approaches $1$ for a model with a dedicated induction
head.  Crucially, the stimulus sequences are generated independently of the
DGP test set, so the score measures a circuit property of the model rather than
a statistical artifact of the training distribution.

{\bf Function vector score.} We operationalize function vectors following~\citet{todd2023function}.  For each test sequence, we designate the tokens from all but
the final segment as the \textit{source context} and the final segment alone as the
\textit{target context}.  We extract the attention output of each layer at the final
position of the source context, then inject it additively into the corresponding
layer's residual stream during a clean forward pass over the target context.
The function vector score for layer $l$ is the mean logit increase for the
correct query token induced by this intervention:
$$
\mathrm{FV}(l) \;=\; \mathbb{E}\!\left[
  f_{\mathrm{patch},l}(X_{< t})[x_t] - f_{\mathrm{base}}(X_{< t})[x_t]
\right],
$$
and we report $\max_l\,\mathrm{FV}(l)$.

In addition to the above-mentioned metrics, we also empirically verify the validity of Assumptions~\ref{ass:low_interference} and \ref{as::directional-concavity}, since these are non-trivial assumptions about the gradient descent dynamics and representation geometry of the model.

{\bf Assumption \ref{ass:low_interference} (Low-Interference Optimization).}  We operationalize this assumption at the granularity of individual attention heads, treating each head
$(l, h)$ as a distinct functional unit $f_{l,h}$ with residual contribution
$$
f_{l,h}(X) \;=\; \bigl(\alpha^{(l,h)} V^{(l,h)}\bigr) W_O^{(h)\top}
  \;\in\; \mathbb{R}^D,
$$
where $W_O^{(h)}$ is the column slice of the output projection corresponding to
head $h$.

Corresponding to a forward pass in which
each head contribution $f_{l,h}$ and the cross-entropy loss incurred is $\cal L$, the gradient
$g_{l,h} = \partial \mathcal{L} / \partial f_{l,h}$ is available
for each head.  We then compute the $LH \times LH$ matrix of pairwise cosine
similarities between flattened gradient vectors.  The primary scalar we track is
$$
\mathrm{Gradient\;decorrelation} \;=\; \min_{(l,h) \neq (l',h')} \bigl|\cos(g_{l,h},\, g_{l',h'})\bigr|,
$$
the \textit{minimum} off-diagonal cosine similarity.  


{\bf Assumption \ref{as::directional-concavity} (Directional Concavity of Readout).}  The readout is
$\phi(v) = \langle g(v),\, w_t \rangle$, where $g$ denotes the composition of
the final layer normalisation and the linear unembedding, and $v$ is the
residual stream at the query position prior to $g$.  For each test sequence, we
collect $v$ (the residual before the final layer norm) and $u_1$ (the attention
output of an early layer at the same position, as a proxy for a redundant latent
estimator), then compute the exact Hessian $H_\phi(v)$ via second-order
automatic differentiation through the final layer norm.  The directional
curvature is
$$
c_i \;=\; u_1^\top H_\phi(v_i)\, u_1.
$$
Assumption \ref{as::directional-concavity} states $\mathbb{E}[c] \leq 0$; we report the mean
directional curvature. 

\subsection{Results}
\begin{figure}
    \centering
    \includegraphics[width=1.\linewidth]{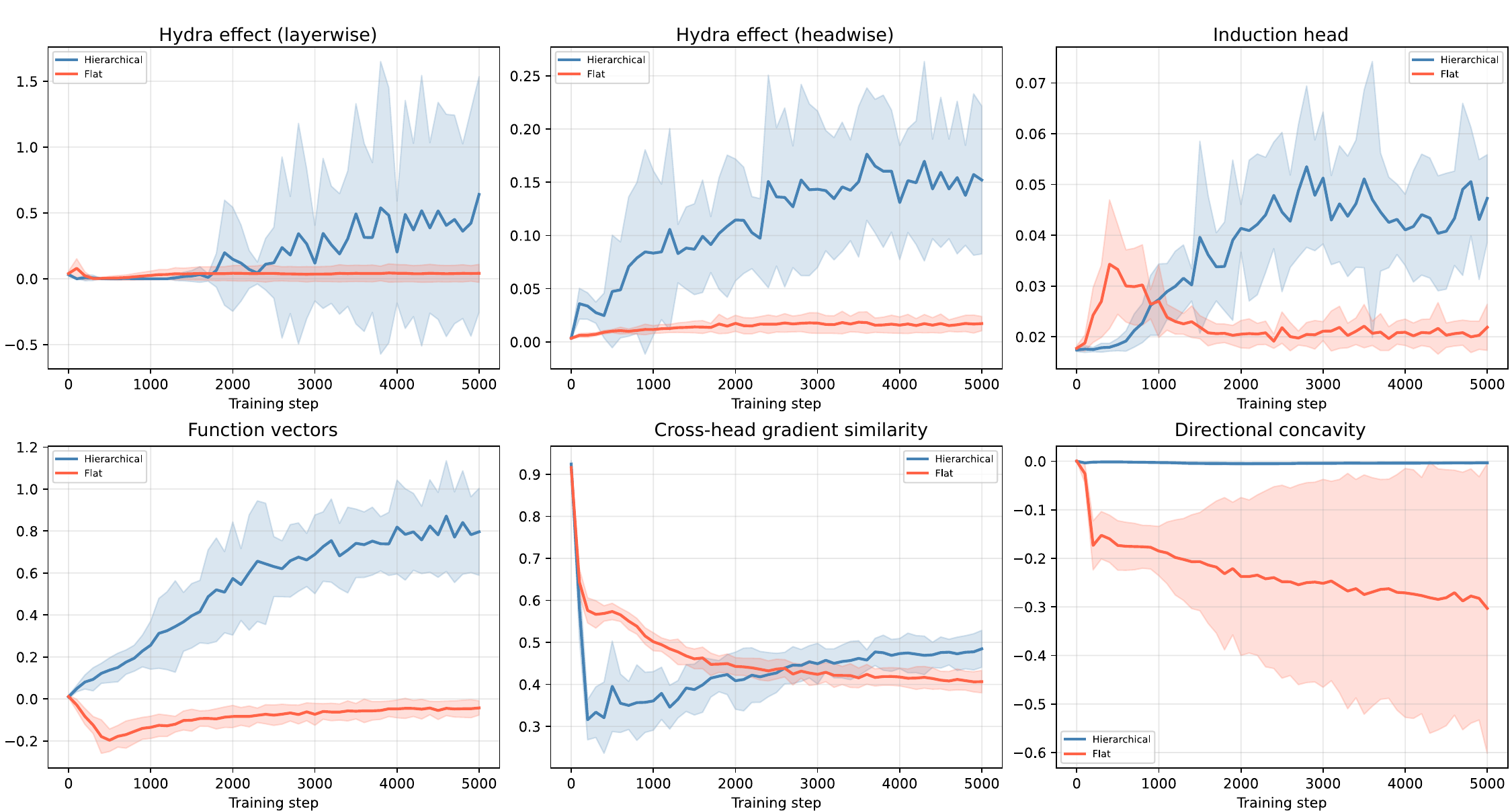}
    \caption{{\bf Emergence of Hydra effect, induction heads, and function vectors in toy models.} All three metrics show selective emergence in the hierarchical DGP against the flat counterpart.}
    \label{fig:toy-model}
\end{figure}

Figure~\ref{fig:toy-model} reports the trajectory of the different metrics in models trained on hierarchical and flat DGPs. Headwise Hydra effect, induction heads, and function vectors emerge
selectively under hierarchical structure.  All three metrics yield
consistently and substantially higher values for the hierarchical DGP across
the entire training run, with
near-zero counterparts in the flat DGP. The distributional separation at the end of training is unambiguous: for all three metrics across all three
aggregation windows (final, peak, AUC), paired t-tests and Wilcoxon
signed-rank tests reject the null at $p<0.001$, and all comparisons survive
Bonferroni correction over the full family of tests (see Figures~\ref{fig:metric-distribution} and \ref{fig:pvalues} in Appendix).
Effect sizes are large throughout, ranging from $d=1.82$ for headwise Hydra
(final) to $d=4.46$ for the function vector AUC, indicating that the
hierarchical condition reliably produces these phenomena across seeds rather
than in isolated runs.

In case of layerwise Hydra effect, the signal is weaker (one-sided significance of $p<0.05$ does not survive Bonferroni correction). We attribute the difference to the coarser resolution
of layer-level ablation: zeroing an entire layer removes the contributions of
multiple heads simultaneously, including heads that serve the same latent stream
and whose ablations may partially cancel in their downstream effect.  Headwise
ablation isolates individual functional units and therefore provides a cleaner
signal of the redundant specialisation predicted by Theorem 4.  The layerwise
metric remains a useful coarse indicator but is a less sensitive operationalisation
of the theoretical construct than its headwise counterpart.

Probing the gradient decorrelation and directional concavity confirms the validity of Assumptions~\ref{ass:low_interference} and \ref{as::directional-concavity} on both model variants. The minimum cross-head gradient cosine similarity drops sharply in the
hierarchical models at approximately the same training step at which the Hydra
and induction scores rise. Note that these two properties of gradient descent and representation geometry, though necessary for Hydra effect to emerge, are not sufficient. In the flat variant, there is no sign of Hydra effect despite a more negative directional curvature and comparable gradient decorrelation. This establishes the absolute necessity of hierarchy as the X-factor to explain induction heads, function vectors, and Hydra effect.

\section{Comparative Validation on Large Models}
\label{sec:large-models}

We design an experiment that isolates the effect of structure in training data while closely matching the surface statistics of natural language. We construct two data generation processes that produce text with similar token distributions, vocabulary usage, and short-range dependencies. One process contains an explicit underlying structure, while the other intentionally lacks such structure. We train identical language models on the resulting corpora using the same architecture, optimization method, and hyperparameters. At fixed intervals of $x$ training steps, we evaluate a predefined set of model properties. We then compare these measurements to those obtained from a model trained on real-world natural language data (OLMo-1B) to validate faithfulness. The pretraining configuration can be found in Appendix~\ref{app:pretraining-details}.
\subsection{Data generation process}
 Each process draws from a fixed vocabulary of size $V$, denoted by the token set $\mathcal{V} = \{0, 1, \ldots, V-1\}$, where \textsc{eos} is a special \underline{e}nd-\underline{o}f \underline{s}equence-symbol. See Appendix~\ref{app:pretraining-details} for full details.

{\bf N-gram.} The N-gram data generation process serves as a baseline because it represents the simplest form of generative structure, one that lacks hierarchy and recursion. It captures only local and sequential dependencies between tokens defined by a fixed history window. Sampling from a Zipf distribution preserves the statistical patterns characteristic of natural language. For an N-gram model of order $n$, the context $h$ consists of the sequence of the previous $n-1$ tokens. For each unique history $h = (w_1, \dots, w_{n-1}) \in \mathcal{V}^{n-1}$, we define a transition distribution over the next token $t$: \begin{equation} P(t | h) \propto \frac{1}{(t + 1)^{\alpha_h}} \end{equation} where the Zipf exponent $\alpha_h$ is sampled as $\alpha_h \sim \max\left(\mathcal{N}(\mu, \sigma^2), \alpha_{\min}\right)$ for each distinct history context. To generate a sentence $\mathbf{x} = (x_1, x_2, \ldots, x_L)$, we first determine the sentence length $L = \ell_{\min} + \epsilon, \quad \text{where } \epsilon \sim \text{Zipf}(s)$. We initialize the history $h_0$ with a set of padding tokens. Then for $i = 1, \ldots, L$: \begin{align} x_i &\sim P(\cdot | h_{i-1}) \ h_i &= \text{shift}(h_{i-1}, x_i) \end{align} where the update step $\text{shift}$ discards the oldest token in the history and appends $x_i$ to form the new context for step $i+1$. Finally, we append the \textsc{eos} token: $x_{L+1} = V - 1$.

{\bf PCSG.} Similar to \citet{Pseudo-PCSG}, we use a context-free backbone (i.e., a PCFG) and introduce context-dependent production rules at the very last stage while expanding non-terminal symbols to terminal symbols. Unlike regular grammars, which capture only local token dependencies, a PCFG introduces a hierarchical structure through nested production rules. This structure enables the generation of sentences with grammatical relationships between constituents, specifically subjects, verbs, and objects. We construct sequences using recursive non-terminal expansions that correspond to different levels (Appendix~\ref{app:cflreference}).  

The data generation process begins with a document symbol \( S \), which expands into multiple document segments \( D_i \). Each segment represents a hierarchical unit, such as a section or paragraph, that further decomposes into sentences. The \texttt{shuffle} operator applied at the document level randomizes the order of these units, preventing the model from overfitting to a fixed sequence order. Within each sentence, the production rules distinguish between two types: assertions and questions. Assertions combine subjects, verbs, and objects into simple or compound statements, while questions invert the order of these components to form interrogative patterns. This hierarchical design captures the syntactic relationships found in natural language, even though the tokens themselves remain abstract.  

Terminal symbols (\(\text{S}_i\), \(\text{V}_i\), and \(\text{O}_i\)) correspond to vocabulary items. Each set of terminals is sampled from a Zipf distribution to reflect the frequency imbalance typical of real linguistic data. Nonterminal symbols like \(\text{Subject}\), \(\text{Verb}\), and \(\text{Object}\) expand into these terminal choices. Context-sensitivity is introduced via permutation functions \(\mathrm{PermuteOrder}_{D_i}\), that assigns different terminal production probabilities across different documents. Formally, for a document $D_i$ and syntactic category $C$, let $\mathcal{V}_C$ be the standard ordered set of terminal symbols. The document-specific vocabulary ordering is defined by applying a random permutation $\pi_{D_i, C}$ to this set:
$$\mathrm{PermuteOrder}_{D_i, C}(\mathcal{V}_C) = \pi_{D_i, C}(\mathcal{V}_C)$$
where $\pi_{D_i, C}$ is a permutation uniformly sampled for each specific document and category. By randomly reordering $\mathcal{V}_C$ before sampling, the stationary rank-probabilities of the Zipf distribution are mapped to different terminal symbols for each document.

As opposed to a PCFG, a PCSG retains context-dependence similar to the n-gram model, and the sole differentiator becomes hierarchy in the latent organization.

\subsection{Metrics}

In this setup, we track the same metrics as Section~\ref{sec:toy-model}. However, to minimize compute, we only track layerwise Hydra effect across the checkpoints. We limit the scope of ablation by focusing on the change of predictive influence of a layer upon the ablation of the immediately preceding layer.

{\bf Parse-Tree Geometry.} To investigate whether the model's internal representation space encodes the hierarchical structure of the data, we employ a structural probe~\citep{hewitt-manning-2019-structural}. We learn a linear transformation $B \in \mathbb{R}^{d_\text{probe} \times d_{\text{model}}}$ such that the squared Euclidean distance between transformed vector representations $h_i, h_j$ approximates the distance between words $x_i, x_j$ in the ground-truth parse tree. Formally, the probe minimizes the deviation between the tree distance $d_{\text{tree}}(x_i, x_j)$ and the probe distance $d_B(h_i, h_j)^2 = (h_i - h_j)^T (B^T B) (h_i - h_j)$. The geometric alignment is quantified using the unlabeled undirected attachment score (UUAS), which measures the percentage of edges shared between the true parse tree and the minimum spanning tree constructed from the predicted distances $d_B$.

\subsection{Results}
\label{subsec:large-results}
\begin{figure}[!t]
    \centering
    \includegraphics[width=0.8\linewidth]{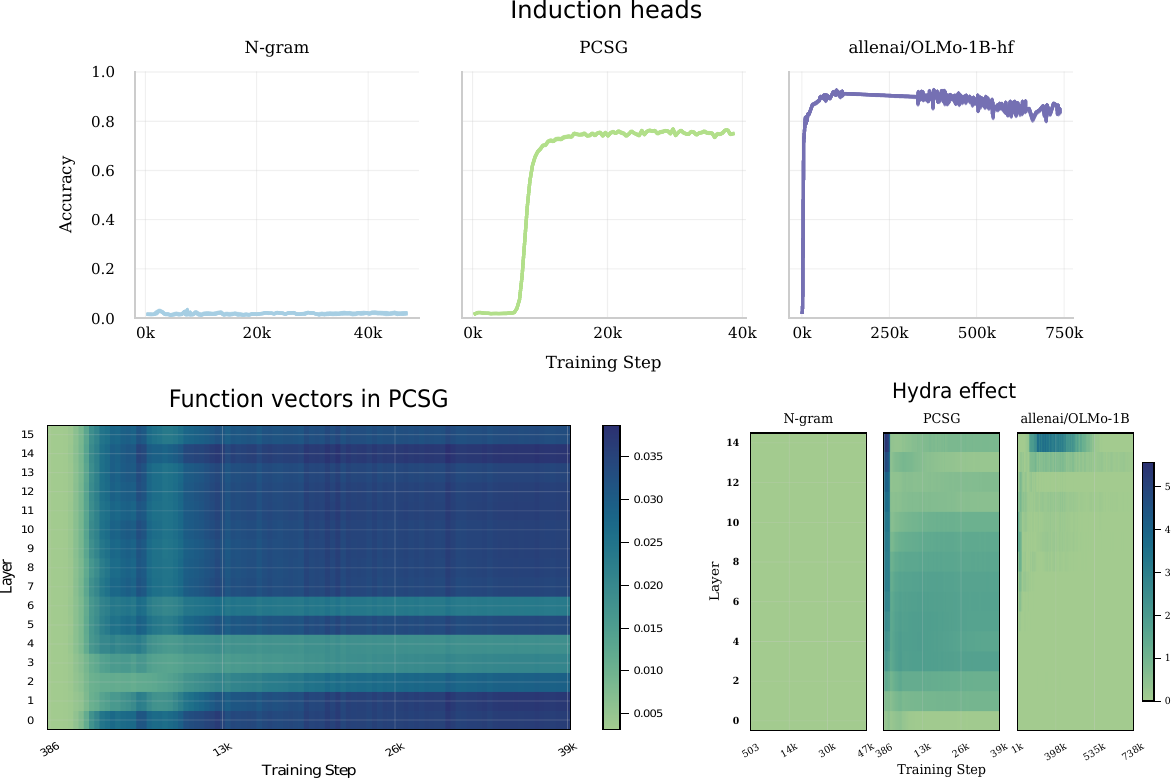}
    \caption{{\bf Development of induction heads, function vectors and Layerwise Hydra effect in synthetic and real models.} We do not show function vectors in n-gram model since it failed to demonstrate any in-context learning, and in OLMo-1B to avoid compute expense.}
    \label{fig:large-model}
\end{figure}


Figure~\ref{fig:large-model} summarizes the main results for across three models and three metrics. Confirming the theory and the toy model experiments, the flat n-gram model does not develop any of the phenomena. In contrast, the  PCSG shows a sharp increase in induction-related attention at approximately 6k training steps. 
Furthermore, its trajectory is comparable to that of the reference model, although it flattens out slightly earlier and is less unstable towards the end of training.
Interestingly, improvement from function vectors substantially increases after around 6k training steps. This is also the point at which the induction heads emerge. N-gram models are not shown due to no observed existence of function vector formation.

PCSG-trained model shows substantial signs of Hydra effect. In fact, the degree of compensation is higher than OLMo-1B. Notice that in the very early stage of training, both models show compensation in the deeper layers. After training, the effect localizes in the middle layers.

\begin{figure}[!t]
    \centering
    \begin{subfigure}[c]{0.48\textwidth}
        \centering
        \includegraphics[
            width=\linewidth,
            height=0.32\textheight,
            keepaspectratio
        ]{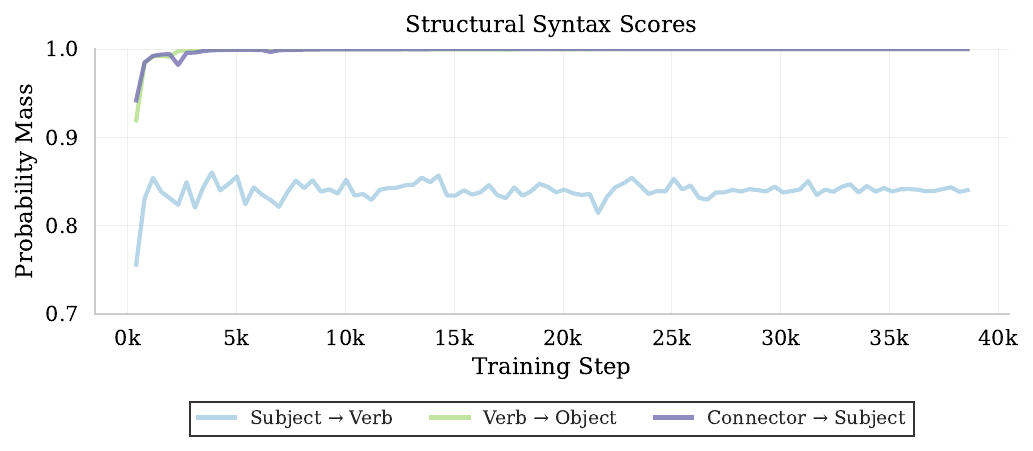}
        \label{fig:parsetree_geo_a}
    \end{subfigure}
    \hfill
    \begin{subfigure}[c]{0.48\textwidth}
        \centering
        \includegraphics[
            width=\linewidth,
            height=0.32\textheight,
            keepaspectratio
        ]{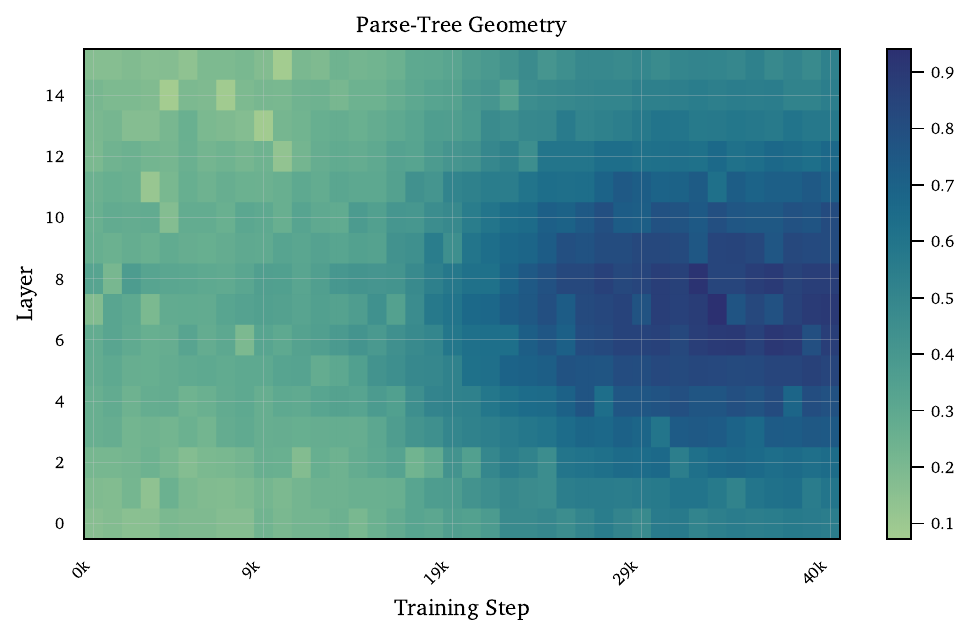}
        \label{fig:parsetree_geo_b}
    \end{subfigure}
\vspace{-5mm}    \caption{\textbf{Hierarchy is internalized in stages during training.} (a) Probability mass towards the next exclusively valid tokens saturates in the beginning, corresponding to shallower hierarchy learning. (b) Layer-wise structural probe accuracy rapidly improves after a substantial training (deep hierarchy learning).}
    \label{fig:parsetree_geo}
\end{figure}

To validate that the model internalizes the generative topology, we first analyze the probability mass assigned to all syntax-valid tokens (Figure~\ref{fig:parsetree_geo}~(a)) and find that this shallow syntax is already learned early in training, at around 4k steps. Next, to get an understanding of the more abstract syntax, we view the geometric alignment between the representation space and the ground-truth parse trees. Figure~\ref{fig:parsetree_geo}~(b) displays the layerwise development of the unlabeled undirected attachment score. We observe that the model progressively learns to map the Euclidean distance of its internal representations to the true tree distance of the generative grammar. This emergence is highly layer-dependent: intermediate layers (layers 5--10) exhibit the strongest structural alignment, reaching a UUAS of approximately 0.9, while earlier and later layers remain less structured. This localization suggests a dedicated syntactic processing stage within the network depth. These two observed shifts in hierarchy learning are also reflected in the loss trajectory (see Figure~\ref{fig:training_curves}): two sharp dips in loss correspond to the emergence of shallow and deep hierarchy representations.

\section{Discussion and Conclusion}

In this work, we seek to unify the development of three mechanistic phenomena, namely induction heads, function vectors, and the Hydra effect, under the lens of the data generation process and optimization dynamics. We identify that hierarchical latent structure in the data generation process can be indicated as a major driver for all three phenomena. Such unification of phenomena that on the surface look very different from each other can be a stepping stone towards building a more robust understanding of language models.
To facilitate a forward-moving discourse, we now highlight the overarching implications of our work.

{\bf Redundancy, interpretability, alignment.} Our theoretical results establish a strong connection between hierarchical structures in the data generation process and redundant distribution of predictive power in the model. Scaling model size and complex reasoning over increasing hierarchical structures, both being the stairway to powerful model, positively reinforce computation in redundancy. This makes intervention-based interpretability (ablate a variable and record its effect) extremely challenging~\citep{DBLP:journals/corr/abs-2602-07930}. This extends to safety alignment as well: unless an alignment method can replace \textit{all} possible realizations of a harmful capability, the model will remain misaligned.

{\bf Compositionality and emergence.} All three phenomena investigated in large model regime demonstrate \textit{emergent optimization}~\citep{lubana2024percolationmodelemergenceanalyzing, NEURIPS2023_9d0f188c} --- a drastic jump from a near-zero to a near-maximum value. Very often this behavior is associated with \textit{phase shift}. Such drastic dynamics have been linked to the compositionality of the associated computation~\citep{arora2023theoryemergencecomplexskills, lubana2024percolationmodelemergenceanalyzing}. This can be linked to learning hierarchical structures as well, both conceptually and empirically. It is natural to assume that the model's estimation of the posterior belief over shallower latents needs to be composed for the subsequent estimation of deeper latents, also evident in the internalized hierarchy in Section~\ref{subsec:large-results}. 

\section*{Acknowledgments}
This work was funded by the LOEWE Distinguished Chair “Ubiquitous Knowledge Processing”, LOEWE initiative, Hesse, Germany (Grant Number: LOEWE/4a//519/05/00.002(0002)/81), as well as by the German Federal Ministry of Education and Research and the Hessian Ministry of Higher Education, Research, Science and the Arts within their joint support of the National Research Center for Applied Cybersecurity ATHENE. Computational resources were provided by the German AI Service Center WestAI.
\bibliography{main}
\bibliographystyle{tmlr}

\appendix

\section{Complete proofs}

\subsection{Proof of Theorem~\ref{thm:additive_evidence}}
\label{app::proof::additive}
\begin{proof}

Starting from the assumed factorization,

\[
p(z^{(0)}, z^{(1)}_{1:m}, x_{1:m})
=p(z^{(0)})
\prod_{i=1}^{m}
p(z_i^{(1)} \mid z^{(0)})
p(x_i \mid z_i^{(1)}, z^{(0)}).
\]

Conditioning on \(z^{(0)}\) and marginalizing the local latents gives

\[
p(x_{1:m}\mid z^{(0)})
=
\sum_{z^{(1)}_{1:m}}
\prod_{i=1}^{m}
p(z_i^{(1)} \mid z^{(0)})
p(x_i \mid z_i^{(1)}, z^{(0)}).
\]

Because the summand factorizes over \(i\), the sum separates:

\[
p(x_{1:m}\mid z^{(0)})
=
\prod_{i=1}^{m}
\left(
\sum_{z_i^{(1)}}
p(z_i^{(1)} \mid z^{(0)})
p(x_i \mid z_i^{(1)}, z^{(0)})
\right).
\]

Applying Bayes' rule,

\[
p(z^{(0)} \mid x_{1:m})
=
\frac{
p(z^{(0)})
p(x_{1:m}\mid z^{(0)})
}{
p(x_{1:m})
}
\]

Substituting the previous expression yields

\[
p(z^{(0)} \mid x_{1:m})
=
\frac{
p(z^{(0)})
\prod_{i=1}^{m}
\left(
\sum_{z_i^{(1)}}
p(z_i^{(1)} \mid z^{(0)})
p(x_i \mid z_i^{(1)}, z^{(0)})
\right)
}{
p(x_{1:m})
}
\]

Taking logarithms gives

\[
\log p(z^{(0)} \mid x_{1:m})
=
\log p(z^{(0)})
+
\sum_{i=1}^{m}
\log
\left(
\sum_{z_i^{(1)}}
p(z_i^{(1)} \mid z^{(0)})
p(x_i \mid z_i^{(1)}, z^{(0)})
\right)
-
\log p(x_{1:m}).
\]

Recognizing the definition of \(e_i\),

\[
e_i(z^{(0)},x_i)
=
\log
\left(
\sum_{z_i^{(1)}}
p(z_i^{(1)} \mid z^{(0)})
p(x_i \mid z_i^{(1)}, z^{(0)})
\right),
\]

we obtain

\[
\log p(z^{(0)} \mid x_{1:m})
=
\log p(z^{(0)})
+
\sum_{i=1}^{m}
e_i(z^{(0)},x_i)
-
\log p(x_{1:m}),
\]

which proves the result.
\end{proof}


\subsection{Proof of Theorem~\ref{thm:fisher_decoupling}}
\label{app::subsec::fisher}

\begin{proof}

Expanding the score decomposition gives

\[
u
=
\sum_i u_i.
\]

Therefore

\[
F
=
\mathbb E
\left[
\left(
\sum_i u_i
\right)
\left(
\sum_j u_j
\right)^\top
\right].
\]

Expanding the product,

\[
F
=
\sum_i
\mathbb E[u_i u_i^\top]
+
\sum_{i\neq j}
\mathbb E[u_i u_j^\top].
\]

Define

\[
F_i
=
\mathbb E[u_i u_i^\top]
\]

and

\[
R
=
\sum_{i\neq j}
\mathbb E[u_i u_j^\top].
\]

It remains to bound $R$.

By Assumption~\ref{ass:low_interference},

\[
|u_i^\top u_j|
\le
\varepsilon
\|u_i\|
\|u_j\|.
\]

Taking expectations,

\[
\left|
\mathbb E[u_i^\top u_j]
\right|
\le
\varepsilon
\mathbb E
\left[
\|u_i\|
\|u_j\|
\right].
\]

Applying Cauchy--Schwarz,

\[
\mathbb E
\left[
\|u_i\|
\|u_j\|
\right]
\le
\sqrt{
\mathbb E[\|u_i\|^2]
\,
\mathbb E[\|u_j\|^2]
}.
\]

Therefore

\[
\left|
\mathbb E[u_i^\top u_j]
\right|
\le
\varepsilon
\sqrt{
\mathbb E[\|u_i\|^2]
\,
\mathbb E[\|u_j\|^2]
}.
\]

Summing over all off-diagonal pairs yields

\[
\|R\|
\le
\varepsilon
\sum_{i\neq j}
\sqrt{
\mathbb E[\|u_i\|^2]
\,
\mathbb E[\|u_j\|^2]
}.
\]

Hence

\[
F
=
\sum_i F_i
+
O(\varepsilon).
\]

Taking the limit $\varepsilon\to0$ gives

\[
F
\rightarrow
\sum_i F_i,
\]

which is block diagonal across evidence streams.

\end{proof}

\subsection{Proof of Theorem~\ref{th::hydra-effect}}
\label{app::proof::hydra}

\begin{proof}
Define the scalar function $\phi(\mu) := \langle g(\mu), w_t \rangle$. Let
\[
\nu := \mu_2 + r,
\]
so that
\[
\Delta = \phi(\nu + \mu_1), \quad \tilde{\Delta}_{k_1} = \phi(\nu).
\]

By Corollary~\ref{cor:specialization} and Assumption~\ref{as::multi-evidence}, the representations $\mu_1$ and $\mu_2$ are functions of disjoint evidence streams $X_{S_1}$ and $X_{S_2}$, respectively, and are conditionally independent given the global latent $Z^{(0)}$. Moreover, both $\mu_1$ and $\mu_2$ serve as estimators of the same latent variable $Z^{(0)}$.

Conditioned on $Z^{(0)}$ and $\nu$, the contribution $\mu_1$ represents an independent estimate of the same underlying signal already partially captured by $\nu$. Therefore, its residual contribution satisfies
\[
\mathbb{E}[\mu_1 \mid Z^{(0)}, \nu] \approx 0.
\]

Using a second-order Taylor expansion of $\phi$ around $\nu$, we obtain
\[
\phi(\nu + \mu_1)
= \phi(\nu)
+ \nabla \phi(\nu)^\top \mu_1
+ \frac{1}{2} \mu_1^\top H_\phi(\nu)\, \mu_1
+ R_3,
\]
where $R_3$ denotes higher-order terms.

Taking conditional expectation given $Z^{(0)}$ and $v$, the first-order term vanishes:
\[
\mathbb{E}[\nabla \phi(\nu)^\top \mu_1 \mid Z^{(0)}, \nu] = \nabla \phi(\nu)^\top \mathbb{E}[\mu_1 \mid Z^{(0)}, \nu] \approx 0
\]

Thus,
\[
\mathbb{E}[\phi(\nu + \mu_1) \mid Z^{(0)}, \nu]
\approx \phi(\nu) + \frac{1}{2} \mathbb{E}\left[\mu_1^\top H_\phi(\nu)\, \mu_1 \mid Z^{(0)}, \nu\right].
\]

By Assumption~\ref{as::directional-concavity}, the second-order term is non-positive:
\[
\mathbb{E}\left[\mu_1^\top H_\phi(\nu)\, \mu_1 \mid Z^{(0)}, \nu\right] \le 0
\]

Hence,
\[
\mathbb{E}[\phi(\nu + \mu_1) \mid Z^{(0)}, \nu] \le \phi(\nu)
\]

Taking expectation over $Z^{(0)}$ and $\nu$, we obtain
\[
\mathbb{E}[\phi(\nu + \mu_1)] \le \mathbb{E}[\phi(\nu)],
\]
i.e.,
\[
\mathbb{E}[\Delta] \le \mathbb{E}[\tilde{\Delta}_{k_1}].
\]

This completes the proof.
\end{proof}

\section{Toy Model Details}
\label{app:toy-models}

\subsection{DGP and Model Configuration}
\label{app:config}

The model uses Rotary Position Embeddings (RoPE) and pre-layer
normalization, with $L_\mathrm{model}$ layers, $H$ attention heads per layer,
and model dimension $D$.  All hyperparameters — architecture, optimizer,
learning rate schedule, and training budget — are held fixed across both DGP
conditions.  The loss mask excludes $[\textsc{Query}]$ cue positions so that
the training signal derives entirely from evidence and query tokens.

\subsection{Statistical Testing}
\label{app:stat-test}

We assess whether each mechanistic effect and assumption probe is
selectively and significantly stronger for models trained on the hierarchical
DGP than for those trained on the flat DGP.

\paragraph{Experimental protocol}  We train $N_\mathrm{seeds}$ independent pairs of
models, one per DGP, sharing the same random seed within each pair but differing
in DGP type.  Each pair constitutes a matched observation, eliminating
seed-level variance from the comparison.  At $E$ evenly spaced evaluation
checkpoints during training, we record all metrics described in
Section~\ref{sec:toy-model}.

\paragraph{Aggregation} For each metric time series $\{m_e\}_{e=1}^E$, we compute
three scalar summaries: (i) the *final* value at the last checkpoint, (ii) the
peak value $\max_e m_e$ over training, and (iii) the area under the curve
(AUC) $\frac{1}{E_\mathrm{final} - E_1} \int m(e)\,\mathrm{d}e$ (approximated
via the trapezoidal rule and normalized by the training span).  Testing all
three windows guards against timing differences in the emergence of each
phenomenon.

For each metric, the hypothesis is:
  $H_1$: hierarchical $>$ flat.

\paragraph{Statistical tests} For each metric–window combination we apply two
one-sided paired tests across the $N_\mathrm{seeds}$ matched pairs: a paired
$t$-test (parametric) and the Wilcoxon signed-rank test (non-parametric).  We
use one-sided tests because the direction of each effect is dictated by theory;
two-sided $p$-values are also reported for completeness.  Effect sizes are
quantified by Cohen's $d$ in the paired formulation,
$d = \bar{\delta} / s_\delta$,
where $\bar{\delta}$ is the mean within-pair difference and $s_\delta$ its
standard deviation.  Multiple-comparison correction is applied via Bonferroni
over all metric–window tests simultaneously; we report both corrected and
uncorrected significance decisions.


\begin{table}[!t]
\centering
\caption{%
  Complete hyperparameter specification for the toy experiments.
  \emph{Hierarchical only} marks parameters that apply exclusively to
  \textsc{HierarchicalDGP}; all other parameters are shared between both
  data generation processes unless otherwise noted.
}
\label{tab:hyperparameters}
\setlength{\tabcolsep}{4pt}
\renewcommand{\arraystretch}{1.15}
%
%
\begin{tabular}{@{} l l @{\hspace{14pt}} l l @{}}
\toprule
\textbf{Parameter} & \textbf{Value} &
\textbf{Parameter} & \textbf{Value} \\
\midrule

\multicolumn{2}{@{}l}{\textit{Data generation (shared)}} &
\multicolumn{2}{l}{\textit{Data generation (hierarchical only)}} \\[2pt]

Vocabulary size $V$          & $64$  &
  Global latent states $K_0$            & $4$        \\

Sequence length $n$          & $128$ &
  Local latent states $K_1$             & $16$       \\

Number of segments $S$       & $16$  &
  $\pi_1$ concentration                 & $30.0$     \\

Segment length $L = n/S$     & $8$   &
  Embedding noise $\sigma$              & $0.1$      \\

Embedding scale              & $4.0$ &
  Query mixing weight $\alpha_q$        & $0.8$      \\

Evidence tokens per segment  & $6$ $(= L - 2)$ &
  Min.\ Bayes gap $\Delta H$            & $0.08$ nats \\

\midrule

\multicolumn{2}{@{}l}{\textit{Model architecture}} &
\multicolumn{2}{l}{\textit{Optimiser}} \\[2pt]

Model dimension $D$          & $64$     &
  Optimiser                              & AdamW     \\

Number of layers             & $6$       &
  Learning rate                          & $3 \times 10^{-4}$ \\

Attention heads $H$          & $8$       &
  $(\beta_1,\,\beta_2)$                 & $(0.90,\; 0.95)$   \\

Head dimension $D/H$         & $8$      &
  Weight decay                           & $0.1$     \\

FFN hidden ratio             & $4\times$ &
  Gradient clipping                      & $1.0$     \\

Position encoding            & RoPE      &
  LR warmup steps                        & $500$     \\

Max.\ context length         & $512$     &
  LR schedule                            & Constant  \\

\midrule

\multicolumn{2}{@{}l}{\textit{Training}} &
\multicolumn{2}{l}{\textit{Evaluation \& significance testing}} \\[2pt]

Training steps               & $5{,}000$              &
  Test set size                          & $512$           \\

Batch size                   & $64$                   &
  Evaluation interval                    & $200$ steps     \\

Sequence sampling            & Online (w/ replacement) &
  Independent seeds $N$                  & $10$            \\

Model vocab size             & $65$ $(V + 1)$         &
  Gradient samples                   & $256$           \\



                             &                        &
  A6 Hessian samples                     & $32$            \\

                             &                        &
  Statistical tests                      & Paired $t$, Wilcoxon \\

                             &                        &
  Correction         & Bonferroni      \\

                             &                        &
  Hypothesis direction                   & One-sided       \\

\bottomrule
\end{tabular}
\end{table}

\begin{figure}[!t]
    \centering
    \includegraphics[width=1.\linewidth]{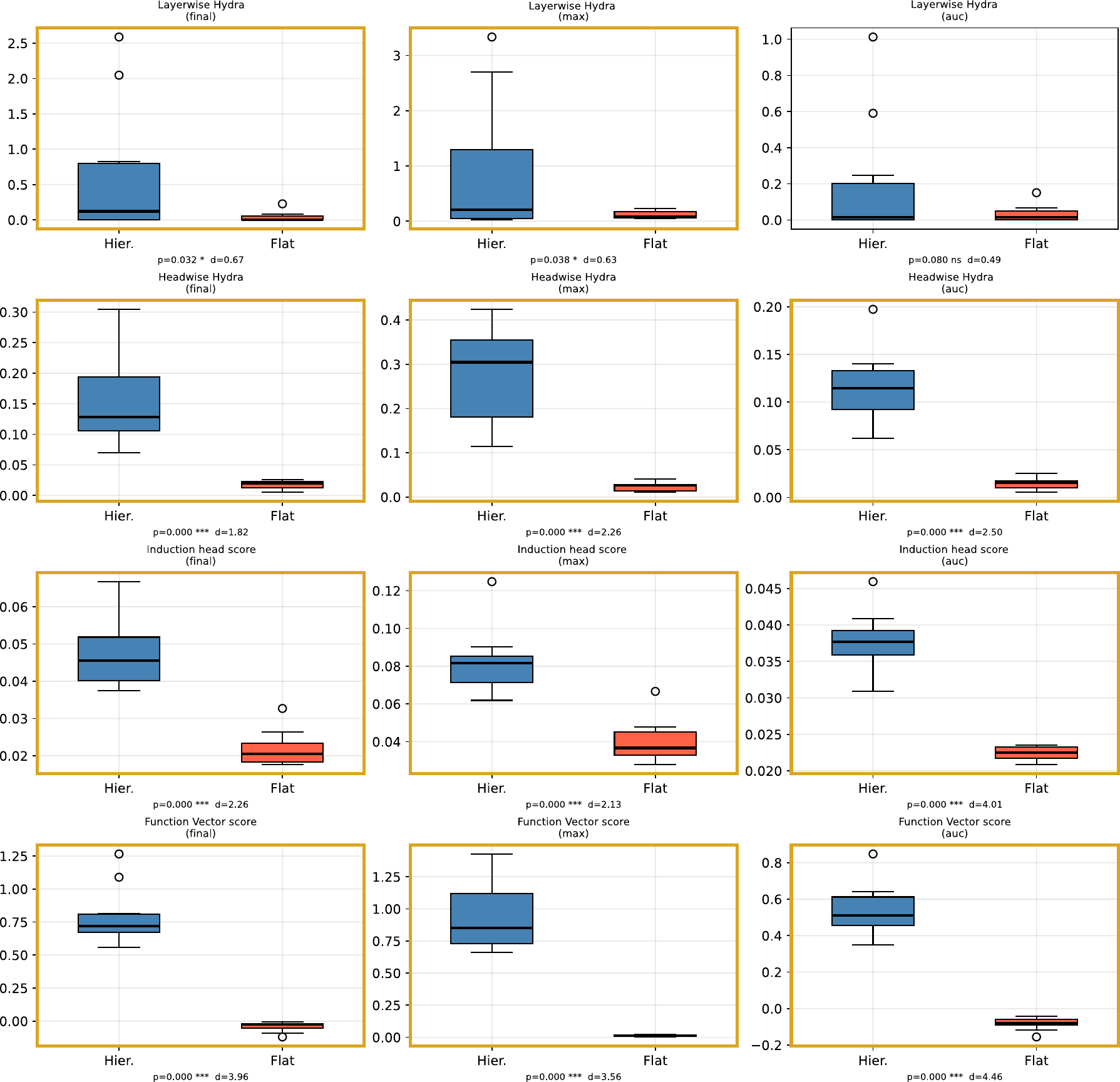}
    \caption{{\bf Distribution of different metrics tracked.} Gold bordered subplots are statistically significant ($p<0.05$ with one-sided t-test).}
    \label{fig:metric-distribution}
\end{figure}
\begin{figure}[!t]
    \centering
    \includegraphics[width=1.\linewidth]{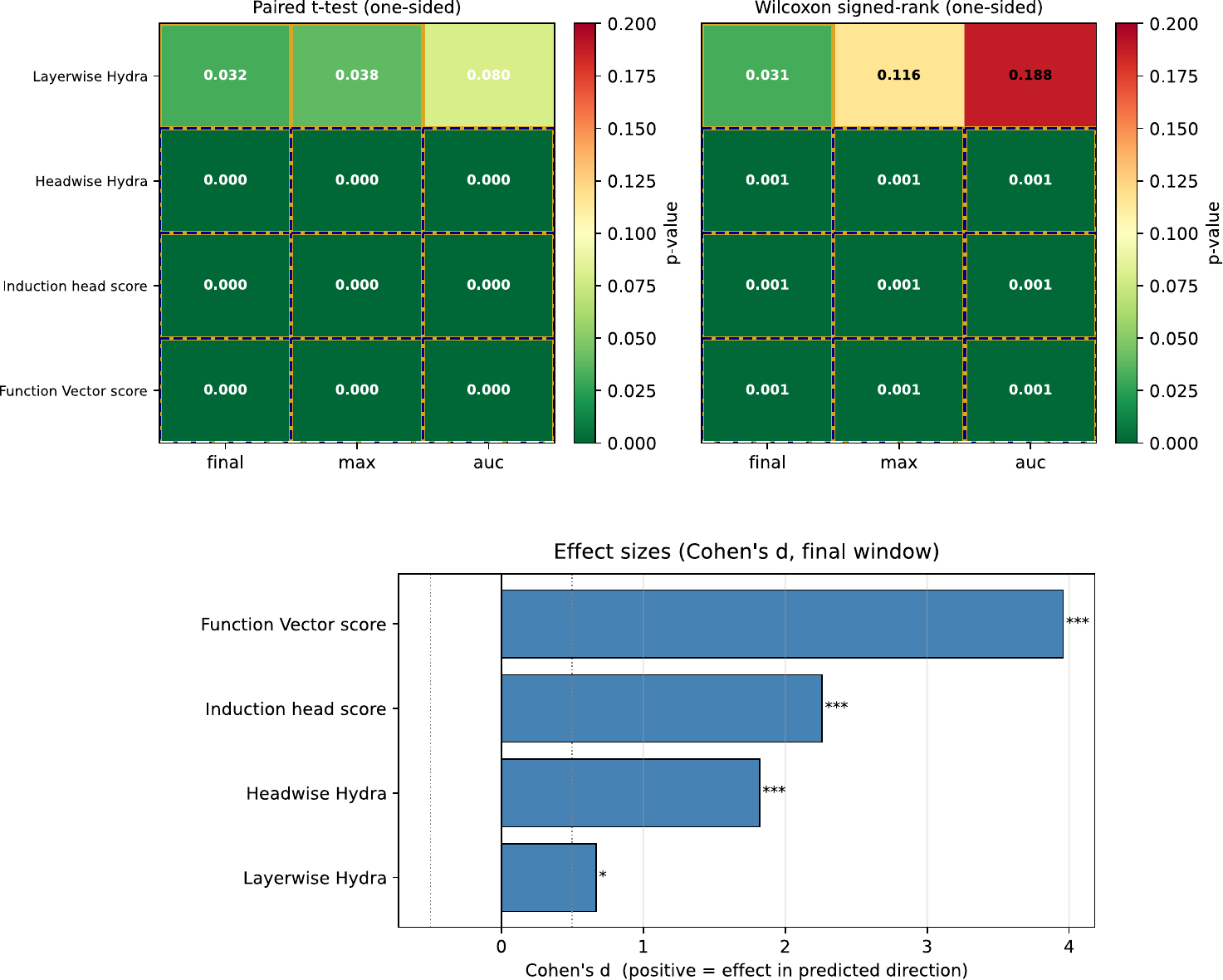}
    \caption{{\bf Significance test results and effect sizes across four metrics.} Gold bordered subplots are statistically significant ($p<0.05$ with one-sided t-test). Navy-dashed subplots survive significance after Bonferroni correction.}
    \label{fig:pvalues}
\end{figure}

\section{Generative Processes Details}

\subsection{N-gram}
\begin{table}[h]
\centering
\begin{tabular}{ll}
\hline
\multicolumn{2}{c}{\textit{Parameters}} \\
\hline
Sentences & 400M \\
Zipf exponent & $\mathcal{N}(2.0, 1.2)$, min 1.2 \\
Length exponent & 2.0 \\
Sequence length & 10–1010 \\
\end{tabular}
\caption{N-gram data generation configuration.}
\label{tab:bigramconfig}
\end{table}
Table~\ref{tab:bigramconfig} shows the parameterization for the N-gram data generation process.
\subsection{PCSG}
\label{app:cflreference}
\begin{table}[h]
\centering
\begin{tabular}{ll}
\hline
\multicolumn{2}{c}{\textit{Parameters}} \\
\hline
Documents & 6.5M \\
Document Repetitions & 10 \\
Sections per Document & 10 \\
Paragraphs per Section & 20 \\
Zipf Exponent & 1.0 \\
Subject Percentage & 0.3 \\
Object Percentage & 0.3 \\
Verb Percentage & 0.3 \\
Connector Percentage & 0.1 \\
\end{tabular}
\caption{PCSG data generation configuration.}
\label{tab:pcfgconfig}
\end{table}

The PCSG production rules are defined in Figure~\ref{fig:grammar} and parameterized as per Table~\ref{tab:pcfgconfig}.

\section{Model \& Training Specifications}
\label{app:pretraining-details}
\begin{table}
\centering
\label{tab:llama2-13b}
\begin{tabular}{ll}
\hline
\multicolumn{2}{c}{\textit{Architecture}} \\
\hline
Parameters & 17M \\
Dimension & 256 \\
Layers  & 16 \\
Attention Heads & 4 \\
Vocabulary Size & 1000 \\
Context Length & 512 \\
\hline
\multicolumn{2}{c}{\textit{Training Details}} \\
\hline
Training Tokens & 10B \\
Learning Rate & $3\times10^{-4}$ \\
Batch Size & 1024 \\
Warmup Steps & 2000 \\
Weight Decay & 0.1 \\
Gradient Clipping & 1.0 \\
\end{tabular}
\caption{LLaMA-based 17M model configuration.}
\end{table}

This appendix details the specifications of our custom model based on the LLaMA 2 architecture~\citep{touvronLlama2Open2023}. It follows the same architectural design but uses adjusted configuration values derived from an informed, proportional downscaling relative to \texttt{allenai/OLMo-1B-hf} to maintain a comparable performance trajectory. We used four H100 GPUs; the total amount required for training was one day. Analysis was done using the \texttt{nnsight} library \citep{fiottokaufman2024nnsightndifdemocratizingaccess}.

\subsection{Loss Trajectories}
\begin{figure*}[!t]
    \centering
    \includegraphics[width=\textwidth]{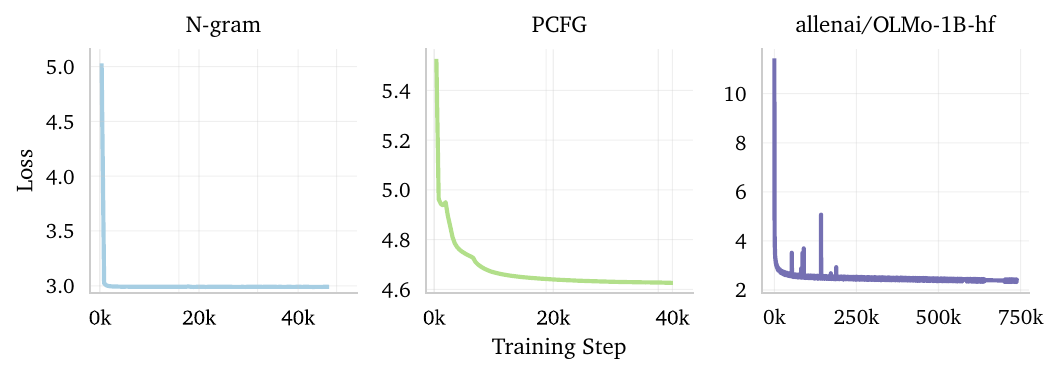}
    \caption{\textbf{Training loss across training.} The models exhibit power-law loss decay throughout training. The N-gram model shows an initial plateau followed by two subsequent drops, whereas the other two models display a steady power-law decline.}
    \label{fig:training_curves}
\end{figure*}

Figure~\ref{fig:training_curves} shows that training loss for the model based on the N-gram data generation, the PCFG data generation, and the reference model all follow a power-law trajectory. Interestingly, the PCFG shows a sharp decrease after an initial flat line. Overall, the Figure shows that the loss trajectories across all three settings are comparable, thereby establishing a baseline for subsequent analyses. Without this quantitative alignment, any proceeding comparison, such as generalization behavior, would be limited, as significant deviations in loss dynamics would indicate that the models had not reached comparable representational regimes. Thus, these results serve as a sanity check confirming that all models exhibit similar convergence behavior.
\begin{figure*}[!t]
\centering
\begin{grammar}
  \firstcase{S}{\mathrm{shuffle}\bigl(\nonterm{D}_1^{A}\;\cdots\;\nonterm{D}_P^{A}\bigr)}{}

  \firstcasesubtil{\ensuremath{\nontermsubtil{D}_i}}{\nontermsubtil{Section}^B}{}

  \firstcasesubtil{\ensuremath{\nontermsubtil{Section}}}{\nontermsubtil{Paragraph}^C}{}

  \firstcasesubtil{\ensuremath{\nontermsubtil{Paragraph}}}{\nontermsubtil{Sentence}^D}{}
  \firstcasesubtil{\ensuremath{\nontermsubtil{Sentence}}}{\nontermsubtil{SentenceType}\; \text{EOS}}{}

  \firstcasesubtil{\ensuremath{\nontermsubtil{SentenceType}}}{\nontermsubtil{Assertion} \gralt \nontermsubtil{Question}}{}

  \firstcasesubtil{\ensuremath{\nontermsubtil{Assertion}}}
    {\nontermsubtil{Subject} \;\nontermsubtil{Verb} \;\nontermsubtil{Object} \gralt \nontermsubtil{Assertion} \; \nontermsubtil{Connector} \; \nontermsubtil{Assertion}}{}

  \firstcasesubtil{\ensuremath{\nontermsubtil{Question}}}
    {\nontermsubtil{Verb} \;\nontermsubtil{Subject} \;\nontermsubtil{Object}}{}

  \firstcasesubtil{\ensuremath{\nontermsubtil{Subject}(D_i)}}
    {\mathrm{PermuteOrder}_{D_i}(\nontermsubtil{Subject})}{}
  \firstcasesubtil{\ensuremath{\nontermsubtil{Subject}}}
    {\mathrm{Zipf}(\text{S\textsubscript{1}} \cdots \text{S\textsubscript{E}}) \longrightarrow
    \text{S\textsubscript{1}} \gralt \cdots \gralt \text{S\textsubscript{E}}}{}

  \firstcasesubtil{\ensuremath{\nontermsubtil{Verb}(D_i)}}
    {\mathrm{PermuteOrder}_{D_i}(\nontermsubtil{Verb})}{}
  \firstcasesubtil{\ensuremath{\nontermsubtil{Verb}}}
    {\mathrm{Zipf}(\text{V\textsubscript{1}} \cdots \text{V\textsubscript{F}}) \longrightarrow
    \text{V\textsubscript{1}} \gralt \cdots \gralt \text{V\textsubscript{F}}}{}

  \firstcasesubtil{\ensuremath{\nontermsubtil{Object}(D_i)}}
    {\mathrm{PermuteOrder}_{D_i}(\nontermsubtil{Object})}{}
  \firstcasesubtil{\ensuremath{\nontermsubtil{Object}}}
    {\mathrm{Zipf}(\text{O\textsubscript{1}} \cdots \text{O\textsubscript{G}}) \longrightarrow
    \text{O\textsubscript{1}} \gralt \cdots \gralt \text{O\textsubscript{G}}}{}

  \firstcasesubtil{\ensuremath{\nontermsubtil{Connector}}}
    {\text{C\textsubscript{1}} \gralt \cdots \gralt \text{C\textsubscript{H}}}{}
\end{grammar}

\begin{minipage}[t]{0.32\textwidth}
\begin{algorithm}[H]
\small
\setlength{\algomargin}{1em}
\DontPrintSemicolon
{%
  \renewcommand{\algorithmcfname}{}%
  \renewcommand{\thealgocf}{}%
  \SetAlgoCaptionSeparator{}%
  \caption{shuffle($X_1,\dots,X_P$)}%
}

$\pi \leftarrow$  $\{1,\dots,P\}$\;

\Return{$X_{\pi(1)},\dots,X_{\pi(P)}$}\;
\end{algorithm}
\end{minipage}
\hfill
\begin{minipage}[t]{0.32\textwidth}
\begin{algorithm}[H]
\small
\setlength{\algomargin}{1em}
\DontPrintSemicolon
{%
  \renewcommand{\algorithmcfname}{}%
  \renewcommand{\thealgocf}{}%
  \SetAlgoCaptionSeparator{}%
  \caption{PermuteOrder$_{D_i}$($X_1,\dots,X_n$)}%
}

$\sigma \leftarrow$ permutation rule for $D_i$\;

\Return{$X_{\sigma(1)},\dots,X_{\sigma(n)}$}\;
\end{algorithm}
\end{minipage}
\hfill
\begin{minipage}[t]{0.32\textwidth}
\begin{algorithm}[H]
\small
\setlength{\algomargin}{1em}
\DontPrintSemicolon
{%
  \renewcommand{\algorithmcfname}{}%
  \renewcommand{\thealgocf}{}%
  \SetAlgoCaptionSeparator{}%
  \caption{\\Zipf($T_1,\dots,T_E,\alpha$)}%
}

Sample $T_k \sim p_k$\;

\Return{$T_k$}\;
\end{algorithm}
\end{minipage}
\caption{Production rules of the PCSG.}
\label{fig:grammar}

\end{figure*}

\end{document}